\documentclass[11pt]{article}

\usepackage{amssymb,amsbsy,amsmath,amscd,amsthm,amsfonts}
\usepackage{cite}
\usepackage{enumerate}
\usepackage{dsfont}
\usepackage{graphicx}
\usepackage{longtable}
\usepackage{tablefootnote}
\usepackage{standalone}
\usepackage{fancyhdr}

\usepackage[margin=1in]{geometry}
\usepackage[bookmarks, colorlinks=true, plainpages = false, citecolor = blue,linkcolor=red,urlcolor = blue, filecolor = blue,pagebackref]{hyperref}

\newtheorem{theorem}{Theorem}

\newtheorem{lemma}{Lemma}

\newtheorem{corollary}{Corollary}

\usepackage{prettyref}
\newrefformat{eq}{(\ref{#1})}
\newrefformat{thm}{Theorem~\ref{#1}}
\newrefformat{sec}{Section~\ref{#1}}
\newrefformat{algo}{Algorithm~\ref{#1}}
\newrefformat{fig}{Fig.~\ref{#1}}
\newrefformat{tab}{Table~\ref{#1}}
\newrefformat{rmk}{Remark~\ref{#1}}
\newrefformat{def}{Definition~\ref{#1}}
\newrefformat{cor}{Corollary~\ref{#1}}
\newrefformat{lmm}{Lemma~\ref{#1}}
\newrefformat{prop}{Proposition~\ref{#1}}
\newrefformat{app}{Appendix~\ref{#1}}
\newrefformat{ex}{Example~\ref{#1}}

\begin{document}

\title{Minimax Lower Bounds for Ridge Combinations Including Neural Nets}

\author{Jason M. Klusowski and Andrew R. Barron \\
Department of Statistics \\ Yale University \\ New Haven, CT, USA \\
Email: \{jason.klusowski, andrew.barron\}@yale.edu}

\maketitle

\begin{abstract}
Estimation of functions of $ d $ variables is considered using ridge combinations of the form $ \textstyle\sum_{k=1}^m c_{1,k} \phi(\textstyle\sum_{j=1}^d c_{0,j,k}x_j-b_k) $ where the activation function $ \phi $ is a function with bounded value and derivative. These include single-hidden layer neural networks, polynomials, and sinusoidal models. From a sample of size $ n $ of possibly noisy values at random sites $ X \in B = [-1,1]^d $, the minimax mean square error is examined for functions in the closure of the $ \ell_1 $ hull of ridge functions with activation $ \phi $. It is shown to be of order $ d/n $ to a fractional power (when $ d $ is of smaller order than $ n $), and to be of order $ (\log d)/n $ to a fractional power (when $ d $ is of larger order than $ n $). Dependence on constraints $ v_0 $ and $ v_1 $ on the $ \ell_1 $ norms of inner parameter $ c_0 $ and outer parameter $ c_1 $, respectively, is also examined. Also, lower and upper bounds on the fractional power are given. The heart of the analysis is development of information-theoretic packing numbers for these classes of functions.
\end{abstract}


\section{Introduction}

Ridge combinations provide flexible classes for fitting functions of many variables. The ridge activation function may be a general Lipschitz function. When the ridge activation function is a sigmoid, these are single-hidden layer artificial neural nets. When the activation is a sine or cosine function, it is a sinusoidal model in a ridge combination form. We consider also a class of polynomial nets which are combinations of Hermite ridge functions. Ridge combinations are also the functions used in projection pursuit regression fitting. What distinguishes these models from other classical functional forms is the presence of parameters internal to the ridge functions which are free to be adjusted in the fit. In essence, it is a parameterized, infinite dictionary of functions from which we make linear combinations. This provides a flexibility of function modeling not present in the case of a fixed dictionary. Here we discuss results on risk properties of estimation of functions using these models and we develop new minimax lower bounds.

For a given activation function $ \phi(z) $ on $ \mathbb{R} $, consider the parameterized family $ \mathcal{F}_m $ of functions
\begin{equation} \label{eq:ridge}
f_m(x) = f_m(x, c_0, c_1, b)  = \textstyle\sum_{k=1}^m c_{1,k} \phi(\textstyle\sum_{j=1}^d c_{0,j,k}x_j-b_k),
\end{equation}
where $ c_1 = (c_{1,1}, \dots, c_{1,m})^{\prime} $ is the vector of outer layer parameters and $ c_{0,k} = (c_{0,1,k}, \dots, c_{0,d,k})^{\prime} $ are the vectors of inner parameters for the single hidden-layer of functions $ \phi(c_{0,k}\cdot x-b_k) $ with horizontal shifts $ b = (b_1,\dots,b_m) $, $ k = 1,\dots,m $. For positive $ v_0 $, let
\begin{equation}
\mathcal{D}_{v_0} = \mathcal{D}_{v_0, \phi} = \{ \phi(\theta \cdot x - t),\; x \in B: \|\theta\|_1 \leq v_0,\; t\in\mathbb{R} \}
\end{equation}
be the dictionary of all such inner layer ridge functions $ \phi(\theta \cdot x - t) $ with parameter restricted to the $ \ell_1 $ ball of size $ v_0 $ and variables $ x $ restricted to the cube $ [-1,1]^d $. The choice of the $ \ell_1 $ norm on the inner parameters is natural as it corresponds to $ \|\theta\|_B = \sup_{x\in B}|\theta\cdot x| $ for $ B = [-1,1]^d $. 

Let $ \mathcal{F}_{v_0, v_1} = \mathcal{F}_{v_0, v_1, \phi} = \ell_1(v_1, \mathcal{D}_{v_0}) $ be the closure of the set of all linear combinations of functions in $ \mathcal{D}_{v_0} $ with $ \ell_1 $ norm of outer coefficients not more than $ v_1 $. These $ v_0 $ and $ v_1 $ control the freedom in the size of this function class. They can either be fixed for minimax evaluations, or adapted in the estimation (as reflected in some of the upper bounds on risk for penalized least square estimation). The functions of the form \prettyref{eq:ridge} are in $ \ell_1(v_1, \mathcal{D}) $ when $ \|c_{0,k}\|_1 \leq v_0 $ and $ \|c_1\|_1 \leq v_1 $. Indeed, let $ \mathcal{F}_{m,v_0,v_1} = \ell_1(m,v_1, \mathcal{D}_{v_0}) $ be the subset of such functions in $ \ell_1(v_1, \mathcal{D}_{v_0}) $ that use $ m $ terms.

Data are of the form $ \{ (X_i, Y_i) \}_{i=1}^n $, drawn independently from a joint distribution $ P_{X,Y} $ with $ P_X $ on $ [-1,1]^d $. The target function is $ f(x) = \mathbb{E}[Y|X=x] $, the mean of the conditional distribution $ P_{Y|X=x} $, optimal in mean square for the prediction of future $ Y $ from corresponding input $ X $. In some cases, assumptions are made on the error of the target function $ \epsilon_i = Y_i - f(X_i) $ (i.e. bounded, Gaussian, or sub-Gaussian). 

From the data, estimators $ \hat{f}(x) = \hat{f}(x,\{ (X_i, Y_i) \}_{i=1}^n) $ are formed and the loss at a target $ f $ is the $ L_2(P_X) $ square error $ \|f-\hat{f}\|^2 $ and the risk is the expected squared error $ \mathbb{E}\|f-\hat{f}\|^2 $. For any class of functions $ \mathcal{F} $ on $ [-1,1]^d $, the minimax risk is 
\begin{equation} \label{eq:minimax}
R_{n,d}(\mathcal{F}) = \inf_{\hat{f}} \sup_{f\in \mathcal{F}} \mathbb{E}\|f-\hat{f}\|^2,
\end{equation}
where the infimum runs over all estimators $ \hat{f} $ of $ f $ based on the data $ \{ (X_i, Y_i) \}_{i=1}^n $.

It is known that for certain complexity penalized least squares estimators \cite{Barron1994}, \cite{Barron2008-2}, \cite{Barron2008}, \cite{Barron1999-2} the risk satisfies
\begin{equation} \label{eq:risk1}
\mathbb{E}\|f-\hat{f}\|^2 \leq \inf_{f_m\in\mathcal{F}_{m}} \{ \|f-f_m\|^2 + \tfrac{cmd\log n}{n} \},
\end{equation}
where the constant $ c $ depends on parameters of the noise distribution and on properties of the activation function $ \phi $, which can be a step function or a fixed bounded Lipschitz function.  The $ d\log n $ in the second term is from the log-cardinality of customary $ d $-dimensional covers of the dictionary. The right side is an index of resolvability expressing the tradeoff between approximation error $ \|f-f_m\|^2 $ and descriptive complexity $ md\log n $ relative to sample size, in accordance with risk bounds for minimum description length criteria \cite{Barron1991}, \cite{Barron2008-3}, \cite{Barron2008-4}, \cite{Barron2014}. When the target $ f $ is in $ \mathcal{F}_{v_1,v_0} $, it is known as in \cite{Lee1992}, \cite{Barron1993}, \cite{Breiman1993} that $ \|f-f_m\|^2 \leq v^2_1/m $ with slight improvements possible depending on the dimension $ \|f-f_m\|^2 \leq v^2_1/m^{1/2+1/d} $ as in \cite{Makovoz1996}, \cite{Klusowski2016}, \cite{Barron1999}. When $ f $ is not in $ \mathcal{F}_{v_0,v_1} $ , let $ f_{v_0,v_1} $ be its projection onto this convex set of functions. Then the additional error beyond $ \|f-f_{v_0,v_1}\|^2 $ is controlled by the bound
\begin{equation} \label{eq:risk2}
\inf_{m} \{ \tfrac{v^2_1}{m} + \tfrac{c_1md\log n}{n} \} = 2v_1(\tfrac{c_1d\log n}{n})^{1/2}.
\end{equation}
Moreover, with $ \hat{f} $ restricted to $ \mathcal{F}_{v_0,v_1} $, this bounds the mean squared error $ \mathbb{E}\|\hat{f}-f_{v_0,v_1}\|^2 $ from the projection. The same risk is available from $ \ell_1 $ penalized least square estimation \cite{Barron2008}, \cite{Barron2008-3}, \cite{Barron2008-4}, \cite{Klusowski2016} and from greedy implementations of complexity and $ \ell_1 $ penalized estimation \cite{Barron2008}, \cite{Klusowski2016}. The slight approximation improvements (albeit not known whether available by greedy algorithms) provide the risk bound \cite{Klusowski2016}
\begin{equation} \label{eq:upperlowdim}
R_{n,d}(\mathcal{F}_{v_0,v_1}) \leq c_2(\tfrac{dv^2_0v^2_1}{n})^{1/2+1/(2(d+1))},
\end{equation}
for bounded Lipschitz activation functions $ \phi $, improving a similar result in \cite{Xiaohong1999}, \cite{Barron1999}.
This fact can be shown through improved upper bounds on the metric entropy from \cite{Mendelson2002}.

A couple of lower bounds on the minimax risk in $ \mathcal{F}_{v_0,v_1} $ are known \cite{Barron1999} and, improving on \cite{Barron1999}, the working paper \cite{Klusowski2016} states the lower bound
\begin{equation} \label{eq:lowerhighdim0}
R_{n,d}(\mathcal{F}_{v_0,v_1}) \geq c_3v^{d/(d+2)}_1(\tfrac{1}{d^4n})^{1/2+1/(d+2))}
\end{equation}
for an unconstrained $ v_0 $.

Note that for large $ d $, these exponents are near $ 1/2 $. Indeed, if $ d $ is large compared to $ \log n $, then the bounds in \prettyref{eq:upperlowdim} and \prettyref{eq:lowerhighdim0} are of the same order as with exponent $ 1/2 $. It is desirable to have improved lower bounds which take the form $ d/n $ to a fractional power as long as $ d $ is of smaller order than $ n $.

Favorable performance of flexible neural network (and neural network like) models has often been observed as in \cite{LeCun2015} in situations in which $ d $ is of much larger order than $ n $. Current developments \cite{Klusowski2016} are obtaining upper bounds on risk of the form
\begin{equation} \label{eq:upperhighdim}
R_{n,d}(\mathcal{F}_{v_0,v_1}) \leq c_4(\tfrac{v^2_0v^4_1\log(d+1)}{n})^{\gamma},
\end{equation}
for fixed positive $ \gamma $, again for bounded Lipschitz $ \phi $. These allow $ d $ much larger than $ n $, as long as $ d = e^{o(n)} $. We have considered two cases. First with greedy implementations of least squares with complexity or $ \ell_1 $ penalty, such upper bounds are obtained in \cite{Klusowski2016} with $ \gamma = 1/3 $ in the noise free case and $ \gamma = 1/4 $ in the sub-Gaussian noise case (which includes the Gaussian noise case). The rate with $ \gamma = 1/3 $ is also possible in the sub-Gaussian noise setting (as well as the noise free setting) via a least squares estimator over a discretization of the parameter space.

It is desirable likewise to have lower bounds on the minimax risk for this setting that show that is depends primarily on $ v^{\alpha}_0v^{2\alpha}_1/n $ to some power (within $ \log d $ factors). It is the purpose of this paper to obtain such lower bounds. Here with $ \gamma = 1/2 $. Thereby, this paper on lower bounds is to provide a companion to (refinement of) the working paper on upper bounds \cite{Klusowski2016}. Lower bounding minimax risk in non-parametric regression is primarily an information-theoretic problem. This was first observed by \cite{Ibragimov1980} and then \cite{Birge1983}, \cite{Birge1986} who adapted Fano's inequality in this setting. Furthermore, \cite{Barron1999} showed conditions such that the minimax risk $ \epsilon^2_n $ is characterized (to within a constant factor) by solving for the approximation error $ \epsilon^2 $ that matches the metric entropy relative to the sample size $ (\log N(\epsilon))/n $, where $ N(\epsilon) $ is the size of the largest $ \epsilon $-packing set. Accordingly, the core of our analysis is providing packing sets for $ \mathcal{F}_{v_0,v_1} $ for specific choices of $ \phi $.

\section{Results for sinusoidal nets}

We now state our main result. In this section, it is for the sinusoidal activation function $ \phi(z) = \sqrt{2}\sin(\pi z) $. We consider two regimes: when $ d $ is larger than $ v_0 $ and visa-versa. In each case, this entails putting a non-restrictive technical condition on either quantity. For $ d $ larger than $ v_0 $, this condition is
\begin{equation} \label{eq:cond1}
\tfrac{d}{v_0}+1 > (c_4\tfrac{v^2_1n}{v_0\log(1+d/v_0)})^{1/v_0},
\end{equation}
and when $ v_0 $ is larger than $ d $,
\begin{equation} \label{eq:cond2}
\tfrac{v_0}{d}+1 > (c_5\tfrac{v^2_1n}{d\log(1+v_0/d)})^{1/d},
\end{equation}
for some positive constants $ c_4, c_5 $.
Note that when $ d $ is large compared to $ \log n $, condition \prettyref{eq:cond2} holds. Indeed, the left side is at least $ 2 $ and the right side is $ e^{\tfrac{1}{d}\log(\tfrac{v^2_1n}{d\log(1+v_0/d)})} $, which is near $ 1 $. Likewise, \prettyref{eq:cond1} holds when $ v_0 $ is large compared to $ \log n $. \\

\begin{theorem} \label{thm:lower}
Consider the model $ Y = f(X) + \varepsilon $ for $ f \in \mathcal{F}_{v_0,v_1, \text{sine}} $, where $ \varepsilon \sim N(0, 1) $ and $ X \sim \text{Uniform}[-1,1]^d $.
If $ d $ is large enough so that \prettyref{eq:cond1} is satisfied, then
\begin{equation} \label{eq:lowerhighdim}
R_{n,d}(\mathcal{F}_{v_0,v_1, \text{sine}}) \geq c_6 (\tfrac{v_0v^2_1\log(1+d/v_0)}{n})^{1/2},
\end{equation}
for some universal constant $ c_6 > 0 $.
Furthermore, if $ v_0 $ is large enough so that \prettyref{eq:cond2} is satisfied, then
\begin{equation} \label{eq:lowerlowdim}
R_{n,d}(\mathcal{F}_{v_0,v_1, \text{sine}}) \geq c_7 (\tfrac{dv^2_1\log(1+v_0/d)}{n})^{1/2}.
\end{equation}
for some universal constant $ c_7 > 0 $.
\end{theorem}

Before we prove \prettyref{thm:lower}, we first state a lemma which is contained in the proof of Theorem 1 (pp. 46-47) in \cite{Gao2013}. \\

\begin{lemma} \label{lmm:subsets}
For integers $ M, L $ with $ M \geq 10 $ and $ 1 \leq L \leq M/10 $, define the set
\begin{equation*}
\mathcal{S} = \{ \omega\in \{0,1\}^M: \|\omega\|_1 = L \}.
\end{equation*}
There exists a subset $ \mathcal{A} \subset \mathcal{S} $ with cardinality at least $ \sqrt{\tbinom{M}{L}} $ such that the Hamming distance between any pairs of $ \mathcal{A} $ is at least $ L/5 $.
\end{lemma}
Note that the elements of the set $ \mathcal{A} $ in \prettyref{lmm:subsets} can be interpreted as binary codes of length $ M $, constant Hamming weight $ L $, and minimum Hamming distance $ L/5 $. These are called constant weight codes and the cardinality of the largest such codebook, denoted by $ A(M, L/5, L) $, is also given a combinatorial lower bound in \cite{Sloane1980}. The conclusion of \prettyref{lmm:subsets} is $ A(M, L/5, L) \geq \sqrt{\tbinom{M}{L}} $.

\begin{proof}[Proof of \prettyref{thm:lower}] For simplicity, we henceforth write $ \mathcal{F}_{v_0, v_1} $ instead of $ \mathcal{F}_{v_0, v_1, \text{sine}} $. Define the collection $ \Lambda = \{ \theta \in \mathbb{Z}^d : \|\theta\|_1 \leq v_0 \} $. Without loss of generality, assume that $ v_0 $ is an integer so that $ M := \#\Lambda = \tbinom{2d+v_0}{2d} $. Consider sinusoidal ridge functions $ \sqrt{2}\sin(\pi\theta\cdot x) $ with $ \theta $ in $ \Lambda $. Note that these functions (for $ \theta \neq 0 $) are orthonormal with respect to the uniform probability measure $ P $ on $ B = [-1, 1]^d $. This fact is easily established using an instance of Euler's formula $ \sin(\pi\theta\cdot x) = \frac{1}{2 i}(\prod_{k=1}^de^{i\pi\theta_k x_k } - \prod_{k=1}^de^{-i\pi\theta_k x_k }) $. 

For an enumeration $ \theta_1, \dots, \theta_M $ of $ \Lambda $, define a subclass of $ \mathcal{F}_{v_0,v_1} $ by
\begin{equation*}
\mathcal{F}_0 = \{ f_{\omega} = \tfrac{v_1}{L}\textstyle\sum_{k=1}^M\omega_k\sqrt{2}\sin(\pi\theta_k\cdot x) : \omega \in \mathcal{A} \},
\end{equation*}
where $ \mathcal{A} $ is the set in \prettyref{lmm:subsets}.
Any distinct pairs $ f_{\omega},f_{\omega^{\prime}} $ in $ \mathcal{F}_0 $ have $ L_2(P) $ squared distance at least $ \|f_{\omega}-f_{\omega^{\prime}}\|^2 \geq v^2_1\|\omega-\omega^{\prime}\|^2_2/L^2 \geq v^2_1/(5L) $. A separation of $ \epsilon^2 $ determines $ L = (v_1/(\sqrt{5}\epsilon))^2 $. Depending on the size of $ d $ relative to $ v_0 $, there are two different behaviors of $ M $. For $ d > v_0 $, we use $ M \geq \tbinom{d+v_0}{v_0} \geq (1+d/v_0)^{v_0} $ and for $ d < v_0 $, $ M \geq \tbinom{d+v_0}{d} \geq (1+v_0/d)^d $.

By \prettyref{lmm:subsets}, a lower bound on the cardinality of $ \mathcal{A} $ is $ \sqrt{\tbinom{M}{L}} $ with logarithm lower bounded by $  (L/2)\log(L/M) $. To obtain a cleaner form that highlights the dependence on $ L $, we assume that $ L \leq \sqrt{M} $, giving $ \log(\# \mathcal{A}) \geq (L/4)\log M $. Since $ L $ is proportional to $ (v_1/\epsilon)^2 $, this condition puts a lower bound on $ \epsilon $ of order $ v_1M^{-1/4} $.
If $ \epsilon > v_1/(1+d/v_0)^{v_0/4} $, it follows that a lower bound on the logarithm of the packing number is of order $ \log N_{d > v_0}(\epsilon) = v_0(v_1/\epsilon)^2\log(1+d/v_0) $. If $ \epsilon > v_1/(1+v_0/d)^{d/4} $, a lower bound on the logarithm of the packing number is of order $ \log N_{v_0 > d}(\epsilon) = d(v_1/\epsilon)^2\log(1+v_0/d) $. Thus we have found an $ \epsilon $-packing set of these cardinalities. As such, they are lower bounds on the metric entropy of $ \mathcal{F}_{v_0,v_1} $.

Next we use the information-theoretic lower bound techniques in \cite{Barron1999} or \cite{Tsybakov2008}. Let $ p_{\omega}(x,y) = p(x)\psi(y-f_{\omega}(x)) $, where $ p $ is the uniform density on $ [-1,1]^d $ and $ \psi $ is the $ N(0, 1) $ density. Then
\begin{equation*}
R_{n,d}(\mathcal{F}_{v_0, v_1}) \geq (\epsilon^2/4)\inf_{\hat{f}}\sup_{f\in\mathcal{F}_0}\mathbb{P}(\|f-\hat{f}\|^2 \geq \epsilon^2),
\end{equation*}
where the estimators $ \hat{f} $ are now restricted to $ \mathcal{F}_0 $. The supremum is at least the uniformly weighted average over $ f \in\mathcal{F}_0 $. Thus a lower bound on the minimax risk is a constant times $ \epsilon^2 $ provided the minimax probability is bounded away from zero, as it is for sufficient size packing sets.
Indeed, by Fano's inequality as in \cite{Barron1999}, this minimax probability is at least
\begin{equation*}
1- \tfrac{\alpha\log(\# \mathcal{F}_0)+\log2}{\log(\# \mathcal{F}_0)},
\end{equation*}
for $ \alpha $ in $ (0,1) $, or by an inequality of Pinsker, as in Theorem 2.5 in \cite{Tsybakov2008}, it is at least
\begin{equation*}
\tfrac{\sqrt{\#\mathcal{F}_0}}{1+\sqrt{\#\mathcal{F}_0}}(1-2\alpha-\sqrt{\tfrac{2\alpha}{\log(\#\mathcal{F}_0)}}),
\end{equation*}
for some $ \alpha $ in $ (0, 1/8) $. These inequalities hold provided we have the following
\begin{equation*}
\tfrac{1}{\#\mathcal{F}_0}\textstyle\sum_{\omega\in\mathcal{A}}D(p^n_{\omega}||q) \leq \alpha\log(\#\mathcal{F}_0),
\end{equation*}
bounding the mutual information between $ \omega $ and the data $ \{(X_i, Y_i)\}_{i=1}^n $, where $ q $ is any fixed joint density for $ \{(X_i, Y_i)\}_{i=1}^n $. When suitable metric entropy upper bounds on the log-cardinality of covers $ \mathcal{F}_{\omega^{\prime}\in\mathcal{A}^{\prime}} := \{ f: \|f-f_{\omega^{\prime}}\| < \epsilon^{\prime} \} $ of $ \mathcal{F}_0 $ are available, one may use $ q $ as a uniform mixture of $ p^n_{\omega^{\prime}} $ for $ \omega^{\prime} $ in $ \mathcal{A}^{\prime} $ as in \cite{Barron1999}, as long as $ \epsilon $ and $ \epsilon^{\prime} $  are arranged to be of the same order. In the special case that $ \mathcal{F}_0 $ has small radius already of order $ \epsilon $, one has the simplicity of taking $ \mathcal{A}^{\prime} $ to be the singleton set consisting of $ \omega^{\prime} = 0 $.  In the present case, since each element in $ \mathcal{F}_0 $ has squared norm $ v^2_1/L = 5\epsilon^2 $ and pairs of elements in $ \mathcal{F}_0 $ have squared separation $ \epsilon^2 $, these function are near $ f_0 \equiv 0 $ and hence we choose $ q = p^n_0 $.
A standard calculation yields
\begin{equation*}
D(p^n_{\omega}||p^n_0) \leq \tfrac{n}{2}\|f_{\omega}\|^2 \leq \tfrac{nv^2_1}{2L} = (5/2)n\epsilon^2.
\end{equation*}

We choose $ \epsilon_n $ such that this $ (5/2)n\epsilon^2_n \leq \alpha\log(\# \mathcal{F}_0) $. Thus, in accordance with \cite{Barron1999}, if $ N_{d > v_0}(\epsilon_n) $ and $ N_{v_0 > d}(\epsilon_n) $ are available lower bounds on $ \#\mathcal{F}_0 $, to within a constant factor, a minimax lower bound $ \epsilon^2_n $ on the $ L_2(P) $ squared error risk is determined by matching
\begin{equation*}
\epsilon^2_n = \tfrac{\log N_{d > v_0}(\epsilon_n)}{n},
\end{equation*}
and
\begin{equation*}
\epsilon^2_n = \tfrac{\log N_{v_0 > d}(\epsilon_n)}{n}.
\end{equation*}
Solving in either case, we find that
\begin{equation*}
\epsilon^2_n = (\tfrac{v_0v^2_1\log(1+d/v_0)}{n})^{1/2},
\end{equation*}
and
\begin{equation*}
\epsilon^2_n = (\tfrac{dv^2_1\log(1+v_0/d)}{n})^{1/2}.
\end{equation*}
These quantities are valid lower bounds on $ R_{n,d}(\mathcal{F}_{v_0,v_1}) $ to within constant factors, provided $ N_{d>v_0}(\epsilon_n) $ and $ N_{v_0 > d}(\epsilon_n) $ are valid lower bounds on the $ \epsilon_n $-packing number of $ \mathcal{F}_{v_0,v_1} $. Checking that $ \epsilon_n > v_1/(1+d/v_0)^{v_0/2} $ and $ \epsilon_n > v_1/(1+v_0/d)^{d/2} $ yields conditions \prettyref{eq:cond1} and \prettyref{eq:cond2}, respectively.
\end{proof}

\section{Implications for neural nets}

The variation of a function $ f $ with respect to a dictionary $ \mathcal{D} $ \cite{Barron1992}, also called the atomic norm of $ f $ with respect to $ \mathcal{D} $, denoted $ V_f(\mathcal{D}) $, is defined as the infimum of all $ v $ such that $ f $ is in $ \ell_1(v,\mathcal{D}) $.  Here the closure in the definition of $ \ell_1(v,\mathcal{D}) $ is taken in $ L_{\infty} $.

Define $ \phi(z) = \sqrt{2}\sin(\pi z) $. On the interval $ [-v_0, v_0] $, it can be shown that $ \phi(z) $ has variation $ V_{\phi} = 2\sqrt{2}\pi v_0 $ with resepct to the dictionary of unit step activation functions $ \pm\text{step}(z^{\prime}-t^{\prime}) $, where $ \text{step}(z) = \mathbb{I}\{z > 0\} $, or equivalently, variation $ \sqrt{2}\pi v_0 $ with respect to the dictionary of signum activation functions with shifts $ \pm\text{sgn}(z^{\prime}-t^{\prime}) $, where $ \text{sgn}(z) = 2\text{step}(z)-1 $. This can be seen directly from the identity
\begin{equation*}
\sin z = \tfrac{v}{2}\int_{0}^1 \cos(v t)[ \text{sgn}(z/v-t) -\text{sgn}(-z/v-t) ]dt,
\end{equation*}
for $ |z| \leq v $. Evaluation of $ \int_{0}^1|\cos(v t)|dt $ gives the exact value of $ \phi $ with respect to $ \text{sgn} $ as $ \sqrt{2}\pi v_0 $ for integer $ v = v_0 $. Accordingly, $ \mathcal{F}_{v_0, v_1, \phi} $ is contained in $ \mathcal{F}_{1, \sqrt{2}\pi v_0v_1, \text{sgn}} $.

Likewise, for the clipped linear function $ \text{clip}(z) = \text{sgn}(z)\min\{ 1, |z| \}  $ a similar identity holds:
\begin{align*}
\sin z & = z + \tfrac{v^2}{2}\int_{0}^1 \sin(v t)[ \text{clip}(-2z/v-2t-1) - \\ & \qquad \text{clip}(2z/v-2t-1) ]dt,
\end{align*}
for $ |z| \leq v $. The above form arises from integrating
\begin{align*}
\cos w & = \cos v - \tfrac{v}{2}\int_{0}^1 \sin(v t)[ \text{sgn}(-w/v-t) + \\ & \qquad \text{sgn}(w/v-t) ]dt,
\end{align*}
from $ w = 0 $ to $ w = z $. And likewise, evaluation of $ \int_{0}^1|\sin(v t)|dt $ gives the exact variation of $ \phi $ with respect to the dictionary of clip activation functions $ \pm\text{clip}(z^{\prime}-t^{\prime}) $ as $ V_{\phi} = \sqrt{2}\pi(v^2_0+1) $ for integer $ v = v_0 $. Accordingly, $ \mathcal{F}_{v_0, v_1, \phi} $ is contained in $ \mathcal{F}_{2, \sqrt{2}\pi(v^2_0+1)v_1, \text{clip}} $ and hence we have the following corollary. \\

\begin{corollary} \label{cor:lowersigmoid}
Using the same setup and conditions \prettyref{eq:cond1} and \prettyref{eq:cond2} as in \prettyref{thm:lower}, the minimax risk for the sigmoid classes $ \mathcal{F}_{1, \sqrt{2}\pi v_0v_1, \text{sgn}} $ and $ \mathcal{F}_{2, \sqrt{2}\pi(v^2_0+1)v_1, \text{clip}} $ have the same lower bounds \prettyref{eq:lowerhighdim} and \prettyref{eq:lowerlowdim} as for $ \mathcal{F}_{v_0,v_1,\text{sine}} $.
\end{corollary}

\section{Implications for polynomial nets}

It is also possible to give minimax lower bounds for the function classes $ \mathcal{F}_{v_0,v_1,\phi_{\ell}} $ with activation function $ \phi_{\ell} $ equal to the standardized Hermite polynomial $ H_{\ell}/\sqrt{\ell!} $, where $ H_{\ell}(z) = (-1)^{\ell}e^{\tfrac{z^2}{2}}\tfrac{d^{\ell}}{dz^{\ell}}e^{-\tfrac{z^2}{2}}  $. 
As with \prettyref{thm:lower}, this requires a lower bound on $ d $:
\begin{equation} \label{eq:cond3}
\tfrac{d}{v^2_0} > (c_8\tfrac{v^2_1n}{v^2_0\log(d/v^2_0)})^{2/v^2_0}.
\end{equation}
for some constant $ c_8 > 0 $. Moreover, we also need a growth condition on the order of the polynomial $ \ell $:
\begin{equation} \label{eq:cond4}
\ell > c_9\log(\tfrac{v^2_1n}{v^2_0\log(d/v^2_0)}),
\end{equation}
for some constant $ c_9 > 0 $. In light of \prettyref{eq:cond3}, condition \prettyref{eq:cond4} is also satisfied if $ \ell $ is at least a constant multiple of $ v^2_0\log(d/v^2_0) $. \\

\begin{theorem} \label{thm:lower_poly}
Consider the model $ Y = f(X) + \varepsilon $ for $ f \in \mathcal{F}_{v_0,v_1, \phi_{\ell}} $, where $ \varepsilon \sim N(0, 1) $ and $ X \sim N(0, I_d) $.
If $ d $ and $ \ell $ are large enough so that conditions \prettyref{eq:cond3} and \prettyref{eq:cond4} are satisfied, respectively, then
\begin{equation} \label{eq:lowerhighdim_poly}
R_{n,d}(\mathcal{F}_{v_0,v_1, \phi_{\ell}}) \geq c_{10}(\tfrac{v^2_0v^2_1\log(d/v^2_0)}{n})^{1/2},
\end{equation}
for some universal constant $ c_{10} > 0 $.
\end{theorem}
\begin{proof}[Proof of \prettyref{thm:lower_poly}]
By \prettyref{lmm:subsets}, if $ d \geq 10 $ and $ 1 \leq d^{\prime} \leq d/10 $, there exists a subset $ \mathcal{C} $ of $ \{0, 1\}^d $ with cardinality at least $ M := \sqrt{\tbinom{d}{d^{\prime}}} $ such that each element has Hamming weight $ d^{\prime} $ and pairs of elements have minumum Hamming distance $ d^{\prime}/5 $. Thus, if $ a $ and $ a^{\prime} $ belong to this codebook, $ |a \cdot a^{\prime}| \leq (9/10)d^{\prime} $. Choose $ d^{\prime} = v^2_0 $ (assuming that $ v^2_0 $ is an integer less than $ d $), and form the collection $ \mathcal{B} = \{ \theta = a/v_0: a \in \mathcal{C} \} $. Note that each member of $ \mathcal{B} $ has unit $ \ell_2 $ norm and $ \ell_1 $ norm $ v_0 $. Moreover, the Euclidean inner product between each pair has magnitude bounded by $ 9/10 $. Next, we use the fact that if $ X \sim N(0, I_d) $ and $ \theta, \theta^{\prime} $ have unit $ \ell_2 $ norm, then $ \mathbb{E}[\phi_{\ell}(\theta\cdot X)\phi_{\ell}(\theta^{\prime}\cdot X)] = (\theta\cdot \theta^{\prime})^{\ell} $.
For an enumeration $ \theta_1, \dots, \theta_M $ of $ \mathcal{B} $, define a subclass of $ \mathcal{F}_{v_0,v_1, H_{\ell}} $ by
\begin{equation*}
\mathcal{F}_0 = \{ f_{\omega} = \tfrac{v_1}{L}\textstyle\sum_{k=1}^M\omega_k\phi_{\ell}(\theta_k\cdot x) : \omega \in \mathcal{A} \},
\end{equation*}
where $ \mathcal{A} $ is the set from \prettyref{lmm:subsets}.
Moreover, since each $ \theta_k $ has unit norm, $ \|\omega-\omega^{\prime} \|_1 \geq L/5 $, and $ \|\omega-\omega^{\prime}\|^2_1 \leq 2L\|\omega-\omega^{\prime}\|_1 $,
\begin{align*}
\|f_{\omega}-f_{\omega^{\prime}}\|^2 
& = \tfrac{v^2_1}{L^2}[\|\omega-\omega^{\prime}\|_1 + \\ & \qquad \textstyle\sum_{i\neq j}(\omega_i-\omega^{\prime}_i)(\omega_j-\omega^{\prime}_j)(\theta_i \cdot \theta_j)^{\ell}] \\
& \geq \tfrac{v^2_1}{L^2}[\|\omega-\omega^{\prime}\|_1 -\|\omega-\omega^{\prime}\|^2_1(9/10)^{\ell}] \\
& \geq \tfrac{v^2_1}{L^2}\|\omega-\omega^{\prime}\|_1(1-2L(9/10)^{\ell}) \\
& \geq \tfrac{v^2_1}{L}(1-2L(9/10)^{\ell}) \\
& \geq \tfrac{v^2_1}{10L},
\end{align*}
provided $ \ell > \tfrac{\log(4L)}{\log(10/9)} $. A separation of $ \epsilon^2 $ determines $ L = (v_1/(\sqrt{10}\epsilon))^2 $. If $ L \leq \sqrt{M} $, or equivalently, $ \epsilon \geq v_1M^{-1/4} $, then $ \log(\#\mathcal{F}_0) $ is at least a constant multiple of $ \log N_{d>v_0}(\epsilon) = (v_0v_1/\epsilon)^2\log(d/v^2_0) $. As before in \prettyref{thm:lower}, a minimax lower bound $ \epsilon^2_n $ on the $ L_2(P) $ squared error risk is determined by matching
\begin{equation*}
\epsilon^2_n = \tfrac{\log N_{d > v_0}(\epsilon_n)}{n},
\end{equation*}
which yields
\begin{equation*}
\epsilon^2_n = (\tfrac{v^2_0v^2_1\log(d/v^2_0)}{n})^{1/2}.
\end{equation*}
If conditions \prettyref{eq:cond3} and \prettyref{eq:cond4} are satisfied, $ N_{d>v_0}(\epsilon_n) $ is a valid lower bound on the $ \epsilon_n $-packing number of $ \mathcal{F}_{v_0, v_1, \phi_{\ell}} $.
\end{proof}

\textbf{Remark}. It is possible to obtain similar lower bounds with $ H_{\ell}(z) $ replaced by a clipped version, in which it is extended at constant height for $ |z| > \zeta_{\ell,\delta} $, where $ \mathbb{E}[\phi^2_{\ell}(Z)\mathbb{I}\{|Z|>\zeta_{\ell,\delta}\}] \leq \delta $ and $ Z \sim N(0, 1) $. Then corollary conclusions follow also for sigmoid classes using the variation of $ \phi_{\ell}(z) $ on $ \{z: |z|\leq\zeta_{\ell,\delta} \} $. Thereby, we obtain lower bounds for sigmoid nets for Gaussian design as well as for the uniform design of \prettyref{cor:lowersigmoid}.

\section{Discussion}

Our risk lower bound of the form $ (\tfrac{v_0v^2_1\log(1+d/v_0)}{n})^{1/2} $ shows that in the very high-dimensional case, it is the $ v_0v_1^2/n $ to a half-power that controls the rate (to within a logarithmic factor). The $ v_0 $ and $ v_1 $, as $ \ell_1 $ norms of the inner and outer coefficient vectors, have the interpretations as the effective dimensions of these vectors. Indeed, a vector in $ \mathbb{R}^d $ with bounded coefficients that has $ v_0 $ non-negligible coordinates has $ \ell_1 $ norm of thin order. These rates confirm that it is the power of these effective dimensions over sample size $ n $ (instead of the full ambient dimension $ d $) that controls the main behavior of the statistical risk. Our lower bounds on packing numbers complement the upper bound covering numbers in \cite{Bartlett1998} and \cite{Klusowski2016}.

Our rates are akin to those obtained by the authors in \cite{Wainwright2011} for high-dimensional linear regression. However, there is an important difference. The richness of $ \mathcal{F}_{v_0,v_1} $ is largely determined by the sizes of $ v_0 $ and $ v_1 $ and $ \mathcal{F}_{v_0,v_1} $ more flexibly represents a larger class of functions. It would be interesting to see if the $ 1/2 $ power in the lower minimax rates in \prettyref{eq:lowerhighdim} could be further improved to match or get near \prettyref{eq:upperhighdim}.

\bibliographystyle{plain}
\bibliography{ridgelowerref}

\end{document}